\documentclass{article} 
\usepackage{nips12submit_e,times}

\usepackage[psamsfonts]{amssymb}
\newcommand{\PP}{\mathbb{P}}
\newcommand{\RR}{\mathbb{R}}
\newcommand{\NN}{\mathbb{N}}
\newcommand{\CC}{\mathbb{C}}
\newcommand{\Sbb}{\mathbb{S}}

\usepackage{amsmath}
\usepackage{hyperref}
\usepackage{amsmath}
\usepackage{times}
\usepackage{graphicx}
\usepackage{color}
\usepackage{rotating}
\usepackage{bbm}
\usepackage{latexsym}
\usepackage{dsfont}
\usepackage{amsthm}
\usepackage{enumerate}
\usepackage{enumitem}
\usepackage{verbatim}
\usepackage{datetime}
\usepackage{url}

\usepackage{wasysym}

\newtheorem{theorem}{Theorem}
\newtheorem{lemma}[theorem]{Lemma}
\newtheorem{corollary}[theorem]{Corollary}
\newtheorem{proposition}[theorem]{Proposition}
\newtheorem{definition}[theorem]{Definition}

\theoremstyle{remark}
\newtheorem{remark}[theorem]{Remark}

\theoremstyle{definition}
\newtheorem{example}[theorem]{Example}

\newcommand{\conv}{\ensuremath{\operatorname{conv}}}

\newcommand{\Ecal}{\mathcal{E}}
\newcommand{\Mcal}{\mathcal{M}}
\newcommand{\Kcal}{\mathcal{K}}

\newcommand{\Xcal}{\mathcal{X}}

\newcommand{\Scal}{\mathcal{S}}

\newcommand{\Ccal}{\mathcal{C}}
\newcommand{\C}{\mathcal{C}}
\newcommand{\Zcal}{\mathcal{Z}}

\newcommand{\R}{{\RR}}
\newcommand{\ra}{\rightarrow}

\newcommand{\supp}{\operatorname{supp}}

\newcommand{\RBM}{\ensuremath{\operatorname{RBM}}}

\newcommand{\DBN}{\ensuremath{\operatorname{DBN}}}

\newcommand{\ZMP}{\operatorname{ZMP}}
\newcommand{\be}{\boldsymbol{\operatorname{e}}}

\newcommand{\Dir}{\operatorname{Dir}}

\newcommand{\ba}{\boldsymbol{\alpha}}

\title{Kernels and Submodels of Deep Belief Networks}

\author{
Guido F. Mont\'ufar \\
Department of Mathematics\\
Pennsylvania State University\\
University Park, PA 16802 \\
\texttt{gfm10@psu.edu} \\
\And
Jason Morton \\
Department of Mathematics\\
Pennsylvania State University\\
University Park, PA 16802 \\
\texttt{morton@math.psu.edu} \\
}

\nipsfinalcopy 

\begin{document}

\maketitle 

\begin{abstract}
We study the mixtures of factorizing probability distributions represented as visible marginal distributions in stochastic layered networks. 
We take the perspective of kernel transitions of distributions, which  gives a unified picture of distributed representations arising from  Deep Belief Networks (DBN) and other networks without lateral connections. 
We describe combinatorial and geometric properties of the set of kernels and products of kernels realizable by DBNs as the network parameters vary. 
We describe explicit classes of probability distributions, including exponential families, that can be learned by DBNs. 
We use these submodels to bound the maximal and the expected Kullback-Leibler approximation errors of DBNs from above depending on the number of hidden layers and units that they contain. 
\end{abstract}

\section{Introduction}
{\em Deep belief networks} (DBNs) are a kind of learning machine introduced originally in~\cite{Hinton2006}. They are used to extract {\em features} from data, often by an unsupervised pretraining step, so their properties as generative models and their expressive  power are also of interest, see~\cite{BengioDelalleauExpressiveDeep11,Hinton:2008,LeRoux2010,Montufar2011}. 
A DBN can be seen as a concatenation of  modules  that implement  kernel transitions (stochastic linear maps) of probability vectors. 
We describe this perspective in Section~\ref{subsection:ZMP}, and the geometry and combinatorics of the set of kernels that DBNs can represent, in Section~\ref{section:geozono}. See Figure~\ref{DBNlayerfig}. 

The deep belief network probability model $\DBN(n_0, n_1,\ldots,n_l)$  with layers of widths $n_0,\ldots,n_l$ is the set  of marginals 
$P(h^0)=\sum_{h^1\in\{0,1\}^{n_1}}\cdots\sum_{h^l\in\{0,1\}^{n_l}}P(h^0,h^1,\ldots,h^l)$ for all $h^0\in\{0,1\}^{n_0}$, of all joint probability distributions on the states of a layered network. 
The top layer has bipartite undirected connections, with subsequent layers bipartite and downward-directed, giving joint unmarginalized probabilities: 
\begin{align}
 P(h^0,h^1, \ldots,h^l) &= \Big(\prod_{k=1}^{l-1} P(h^{k-1}|h^{k})\Big) P(h^{l-1},h^l)   \;, \label{jointDBN}
\intertext{for all $(h^0,\ldots,h^l)\in\{0,1\}^{n_0}\times\cdots\times\{0,1\}^{n_l}$, where}
\phantom{\text{and}} P(h^{l-1},h^l) &= \frac{1}{Z}\exp\big( h^l  B^l  +   h^l  W^{l} h^{l-1}   +  B^{l-1}  h^{l-1} \big)\;, \text{ and}\label{dbnrbm} \\
P(h^{k-1}|h^{k}) &= \frac{1}{Z_{h^{k}}}\exp\big( h^{k} W^{k} h^{k-1}  +  B^{k-1} h^{k-1} \big)\;.\label{dbncondi}
\end{align}
Here $h^k =  (h^k_1,\ldots,h^k_{n_k}) \in \{0,1\}^{n_k}$ denotes the states of the units in the $k$th layer;  
$W^{k}\in\RR^{n_k\times n_{k-1}}$ is a matrix of  connection weights  between units from the $k$th and $(k-1)$th layer;  
 $B^k\in{\RR}^{n_k}$ is a vector of   bias weights  of the  units in the $k$th layer; 
   $Z=\sum_{h^{l-1}, h^l}\exp(h^{l}W^l h^{l-1} + B^{l-1}  h^{l-1} + b^l h^l )$ is a normalization constant that depends on $W^l,B^{l-1},B^l$; 
and $Z_{h^{k+1}}=\sum_{h^k}\exp( h^{k+1}  W^{k+1} h^{k} + B^{k}  h^{k})$ is a normalization constant that depends on $W^{k+1},B^k$, and $h^{k+1}$. 
The total number of parameters of this model is  $d=(\sum_{k=1}^{l} n_{k-1}  n_k) + (\sum_{k=0}^l n_k)$, treating the layer widths $n_0,\ldots,n_l$ as hyperparameters. 

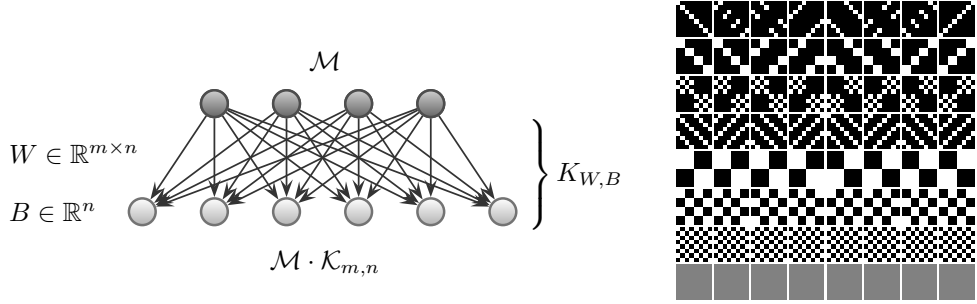
\begin{figure}
\begin{center}
\setlength{\unitlength}{.7cm}
\begin{picture}(8.8,5)(-0.4,-1.2)
\put(0,0){\includegraphics[trim=5.4cm 22.35cm 11.5cm 4cm, clip=true, width=8\unitlength]{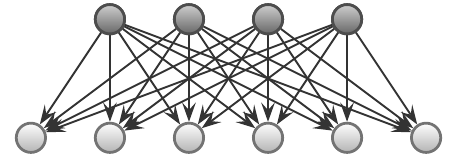}} 
\put(-2,1.25){$W\in\R^{m\times n}$}
\put(-2,0.15){$B\in\R^{n}$}
\put(0,3.05){\begin{minipage}{8\unitlength}\center$\Mcal$\end{minipage}}
\put(0,-.8){\begin{minipage}{8\unitlength}\center$\Mcal\cdot\Kcal_{m,n}$\end{minipage}}
\put(7.3,.9){$\left. \begin{array}{c} {}\\ {}\\ {}\\ {}\end{array}\right\} K_{W,B}$}
\end{picture}
\setlength{\unitlength}{.6cm}
\begin{picture}(7.5,6.4)(0,0)
\put(2,-.3){\includegraphics[trim=7.5cm 10.9cm 7.4cm 5cm, clip=true,width=7.5\unitlength]{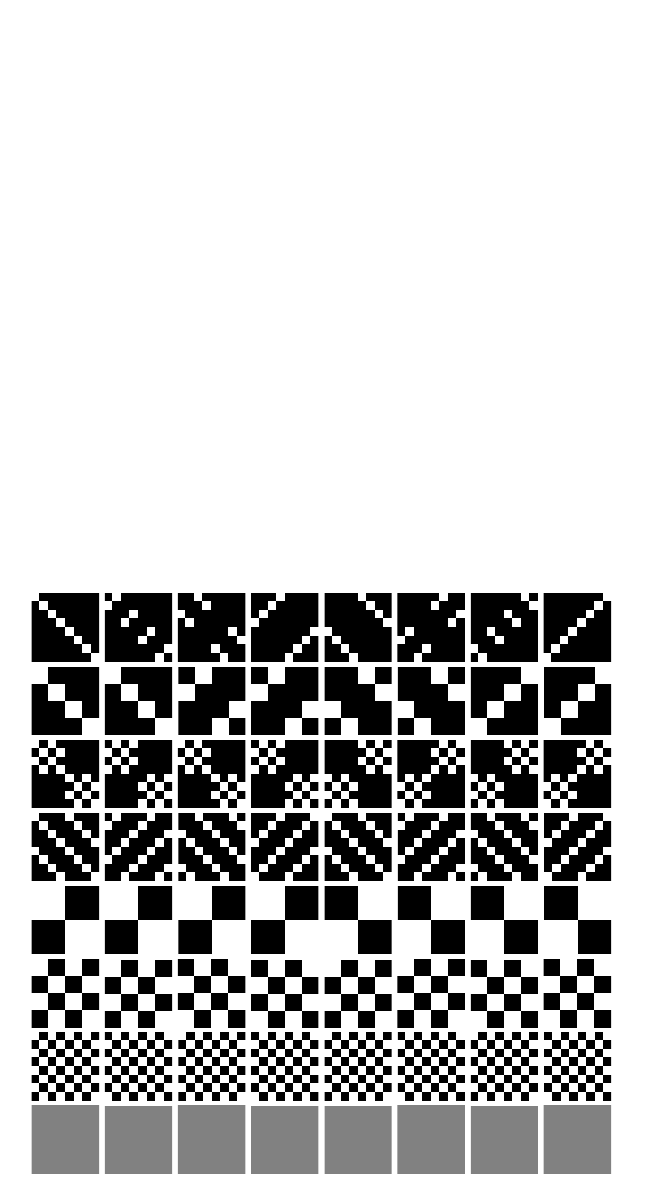}}
\end{picture}
\end{center}
\caption{Left: A network module that realizes stochastic transitions $\Kcal_{m,n}$ from the set of distributions $\Mcal\subseteq\Delta_{2^m-1}$ on  the top layer, to probability distributions  $\Mcal\cdot\Kcal_{m,n}\subseteq\Delta_{2^n-1}$ on the bottom layer, see eq~\eqref{zonosetkerneldef}. 
Right: The kernels $K_{W,B}=\operatorname{K}_p$ in $\Kcal_{3,3}\subset\R^{8\times8}$ described in Proposition~\ref{markovkernels}. 
}\label{DBNlayerfig}
\end{figure}

A {\em restricted Boltzmann machine} (RBM)~\cite{Smolensky1986,Freund1992,Hinton2002} is formally the same as a DBN with only one hidden layer. The model $\RBM_{n,m}=\DBN(n,m)$ is the set of probability distributions on $\{0,1\}^n$ of the form $
 P(v) = \frac{1}{Z}\sum_{h\in\{0,1\}^m} \exp\big(  h  W v  +   C h  +  B  v \big)$ for all $v\in\{0,1\}^n$. 

We denote by $\Delta_{2^n-1}$ the simplex of probability distributions on $\{0,1\}^n$. 
Its vertices are the point measures $\delta_x$, $x\in\{0,1\}^n$. 

Sutskever and Hinton~\cite{Hinton:2008} showed that a very deep and narrow DBN, with $\sim 3\cdot 2^n$ hidden layers of width $(n+1)$, can approximate any distribution on $\{0,1\}^n$ arbitrarily well. 
Le~Roux and Bengio~\cite{LeRoux2010} improved this bound showing that $\sim \frac{2^n}{n}$ layers of width $n$ suffice. 
Mont\'ufar and Ay~\cite{Montufar2011} improved that bound again to $\sim\frac{2^n}{2n}$.  
We are interested in the expressive power of DBNs which have less than $2^n-1$ parameters  and  cannot approximate every probability distribution arbitrarily well. 
In~\cite{NIPS2011_0307} the maximal Kullback-Leibler approximation errors of RBMs were bounded from above by studying submodels of RBMs.

\begin{definition}
A {\em submodel} of a DBN with layer widths $n_0,\ldots,n_l$ is a set of probability distributions in $\Delta_{2^{n_0}-1}$ contained in $\DBN(n_0,\ldots,n_l)$. 
\end{definition}

Approaches to find explicit submodels of DBNs include studying
\vspace{-.2cm}
\begin{itemize}[leftmargin=1.3em]\setlength{\itemindent}{0em}
\item
The set $\DBN(n_0,\ldots,n_l)$ as a mixture of conditional distributions with mixing distributions from the imbedded model $\DBN(n_1,\ldots,n_l)$. 
This approach was proposed in~\cite{Montufar2010} and used in~\cite{MontufarMorton2012} to study the expressive power of RBMs. 
In Section~\ref{subsection:ZMP} we describe distributed mixtures of product distributions arising in layered networks. 
\item
Models arising from {\em probability sharing} on RBMs. This idea has been used in~\cite{Hinton:2008,LeRoux2010,Montufar2011} to study universal approximation of probability distributions by DBNs. 
To study submodels of DBNs, one imposes constraints on the number and type of sharing steps (the number and widths of the hidden layers). The submodels are sub-simplicial-complexes of $\Delta_{2^n-1}$. 
In Section~\ref{sharing} we discuss certain faces of the probability simplex that can be represented by deep and narrow DBNs. 
\item
The set of joint probability distributions on the states of all units of a DBN and their linear projections (by marginalization maps). 
\item
Graphical submodels of the DBN such as RBMs and trees.
\end{itemize}
Understanding these items is helpful to lower bound the capabilities of deep belief networks. 

The marginal probability distributions on the states of the visible units of  a stochastic network  with no direct connections between visible units, are mixtures of product distributions. 
We call a mixture {\em distributed} when the mixture components share parameters in some way. 
Distributed  representations have been  discussed in \cite{Hinton99productsof,Bengio-2009,MontufarMorton2012}. 
Each layer of a DBN defines  a distributed mixture of product distributions. 
Similarly, each layer of a {\em deep Boltzmann machine} (DBM) and a {\em directed RBM} define a distributed mixture of product distributions. 
A DBM is a layered network with undirected bipartitie connections between units in subsequent layers, see~\cite{salakhutdinov2009deep}. 
The DBM model is the set of marginal distributions on the states of the variables in the bottom layer. 
The model $\RBM_{n,m}^{\text{dir}}$ is the set of visible distributions of a  pair of layers of binary units with  directed connections from the top layer to the bottom layer, including top and bottom bias weights, and without connections within each layer, as shown in Figure~\ref{DBNlayerfig}.

In Section~\ref{subsection:ZMP} we discuss the mixtures of product distributions represented by layered networks. 
In Section~\ref{section:geozono} we study the geometry of the set of all stochastic transitions that can be realized by DBN layers. 
In Section~\ref{section:mainresults} we derive upper bounds on the maximal and mean approximation errors of DBNs. 
Section~\ref{section:discussion} presents a discussion of our results. 
All formal proofs of mathematical statements are deferred to the Appendix.

\begin{figure}
\begin{center}
\setlength{\unitlength}{.7cm}
\begin{picture}(6,5.4)(0,0)
\put(0,-.55){\includegraphics[trim=7cm 9.2cm 6.5cm 9cm, clip=true, width=5\unitlength]{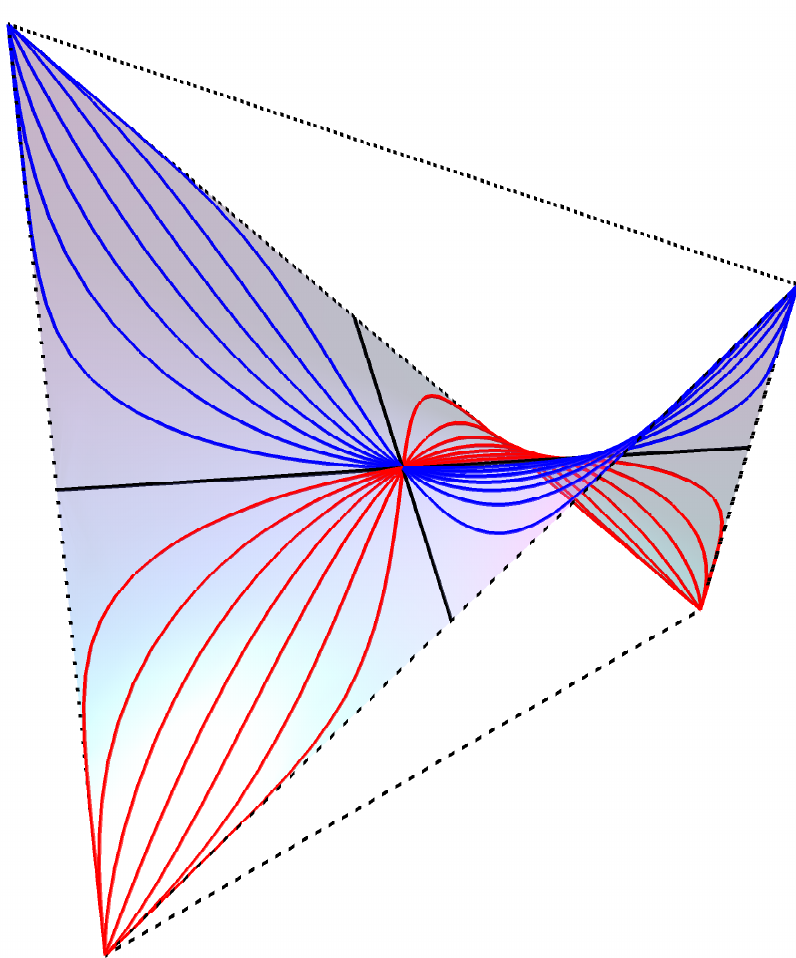}} 
\end{picture}
\begin{picture}(6,5)(0,0)
\put(5,0){\begin{rotate}{90}\includegraphics[trim=7.5cm 10.5cm 5.5cm 10.2cm, clip=true, height=5\unitlength]{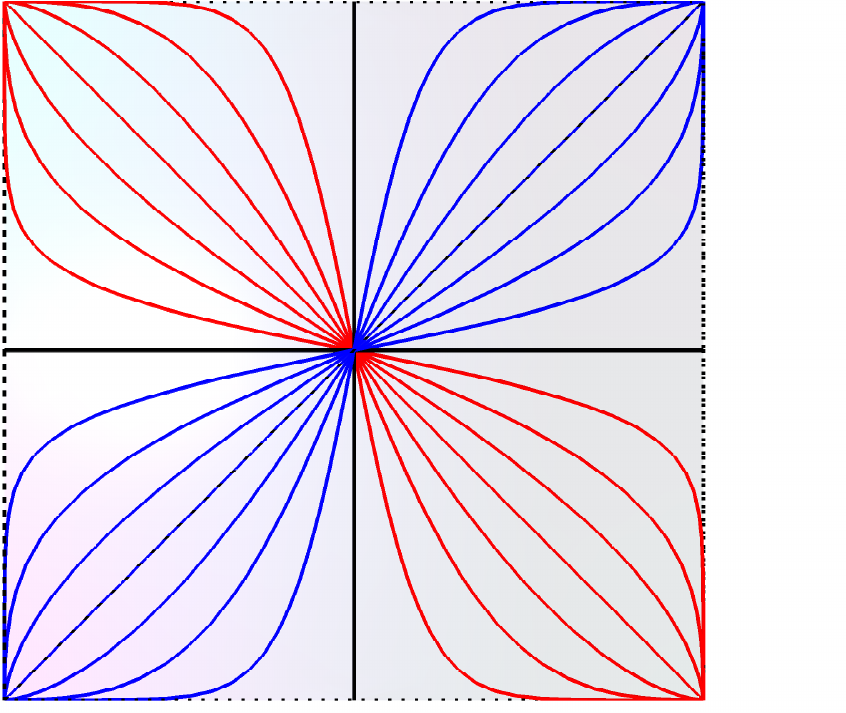}\end{rotate}} 
\end{picture}
\begin{picture}(5,5)(0,0)
\put(0,-.05){\includegraphics[trim=8.25cm 11.35cm 7.9cm 11cm, clip=true, width=5\unitlength]{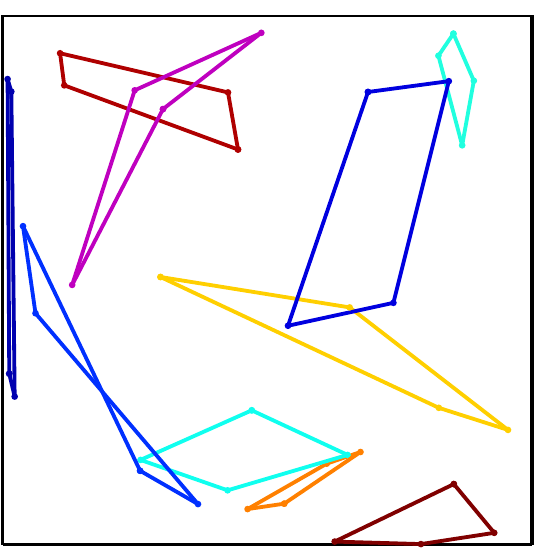}} 
\end{picture}
\end{center}
\caption{Left: Two-bit product distributions with straight lines $\alpha B$, $\alpha\in\RR$ as natural parameters for various choices of $B\in\RR^2$.  Middle: Linear projection of the left figure into the convex support of the two-bit independence model. 
Right: Linear projection of $10$ $(2,2)$-zonoset tuples of product distributions with random  $W\in\R^{2\times2},B\in\R^2$. 
}\label{zonosetsinconvsupp}
\end{figure}

\section{Distributed mixtures of products and stochastic kernels}\label{subsection:ZMP}

An exponential family is a set of probability distributions of the form $\Ecal_V=\{p\propto\exp(f)\colon f\in V\}$, where $V$ is an affine space of functions on the set of elementary events. 
The set of all strictly positive product distributions of $n$ binary variables is an $n$-dimensional exponential family, denoted by $\Mcal_n\subseteq\Delta_{2^n-1}$, with elements $p_B(v_1,\ldots,v_n)=\prod_{i=1}^n p_{B_i}(v_i)=\exp(  B  v  )/Z_B$, $Z_B=\sum_{v\in\{0,1\}^n}\exp(  B  v )$. Here  $B\in\R^n$ is called the {\em natural parameter} vector of $p_B$. 
The convex support of this model is an $n$-dimensional hypercube with points in one-to-one correspondence with the points in the closure $\overline{\Mcal_n}$ of $\Mcal_n$. See~\cite{Brown86:Fundamentals_of_Exponential_Families}. 

The {\em $k$-mixture} of product distributions of $n$ binary variables is
\begin{equation}
\Mcal_{n,k}:=\{\sum_{j=1}^k\lambda_j p^{(j)}\colon p^{(j)}\in\Mcal_n, \lambda_j\geq0 \;\forall j, \text{ and }\sum_{j=1}^k\lambda_j=1 \} \;. 
\end{equation}
This set has the dimension expected from counting parameters, $\dim(\Mcal_{n,m})=\min\{2^n-1, m n +m-1\}$, unless $n=4$ and $m=3$, see~\cite{Catalisano2011}. 

The marginal visible probability distributions of DBNs, DBMs, directed RBMs, and RBMs with $n$ binary visible units and $m$ binary units in the first hidden layer, all have the following form:   
\begin{gather}
p(v)=\sum_{h\in\{0,1\}^{m}} p_{hW+B}(v)\,q(h)\quad \forall v\in\{0,1\}^n,\quad\text{where}  \label{zonodbn}\\
p_{hW+B}(v) = \frac{1}{Z_{h}}\exp( ( h W +B ) v ) \quad \forall v\in\{0,1\}^n,\quad\forall h\in\{0,1\}^m\;, \label{zonoprod}
\end{gather} 
with $Z_h=\sum_{v\in\{0,1\}^n}\exp(( h  W+B) v )$, $W\in{\RR}^{m\times n}$, $B\in{\RR}^{n}$,  
and $q$ is a probability distribution on $h\in\{0,1\}^m$.

The natural parameters $\Zcal=\{ h W +B\colon h\in\{0,1\}^m\}$,  with  $W\in\R^{m\times n}$ and $B\in\R^n$, of the $2^m$ product distributions $\{p_{hW+B}\colon h\in\{0,1\}^m\}$, are a multiset (a set with repetitions allowed) of points in $\R^n$ called an  $(m,n)$-{\em zonoset}. 
In the literature of polytopes the convex hull of a zonoset is known as {\em zonotope}. 
\begin{definition}
We call $\{p_{hW+B}\colon h\in\{0,1\}^m\}$  the {\em zonoset tuple} of product distributions  associated to the zonoset $\Zcal=\{hW+B\colon h\in\{0,1\}^m\}$. 
\end{definition}
The number of parameters of a zonoset tuple is $(m+1)n$, while $2^m n$ parameters are needed for describing an arbitrary tuple of $2^m$ product distributions. 
Any $(m,n)$-zonoset-tuple of product distributions is contained in an exponential subfamily of $\Mcal_{n}$ of dimension $\min\{m,n\}$. 
Figure~\ref{zonosetsinconvsupp} illustrates zonoset tuples of product distributions on $\{0,1\}^2$.

We can view eq.~\eqref{zonodbn} as a transition of the marginal distribution $q$ on the states of the first hidden layer, to the visible distribution $p$, by a stochastic kernel: 
\begin{gather}
p = q \cdot K_{W,B}\;, 
\intertext{where the kernel, called an {\em $(m,n)$-zonoset kernel}, is defined by the $2^m\times2^n$-matrix with entries}
K_{W,B}(h,v):= p_{hW+B}(v)\quad\text{ for all $h\in\{0,1\}^m$ and all $v\in\{0,1\}^n$}\;. \label{zonosetkerneldef}
\end{gather}
Thus a zonoset tuple is the rows of a zonoset kernel viewed as a set.  
Each $K_{W,B}$ is a (row) stochastic matrix describing a linear map 
\begin{equation*}
K_{W,B}\colon \Delta_{2^m-1}\to \conv\{K_{W,B}(h,\cdot)\}_h\subseteq \Delta_{2^n-1}\;;\; q\mapsto p\cdot K_{W,B}\;.
\end{equation*}  
We denote the set of all  $(m,n)$-zonoset kernels  by 
\begin{equation*}
\Kcal_{m,n}:=\{K_{W,B}\colon W\in\R^{m\times n}, B\in\R^n\}\;.
\end{equation*} 
We write $\overline{\Kcal_{m,n}}$ for the set of all kernels that can be expressed as the limit of a sequence $K_{W_i,B_i}\in\Kcal_{m,n}$, $i\in\NN$. 

\begin{figure}
\begin{center}
\setlength{\unitlength}{.6cm}
\begin{picture}(9,7.4)(0,0)
\put(0,0){\includegraphics[trim=6cm 10cm 6cm 10cm, clip=true, width=8\unitlength]{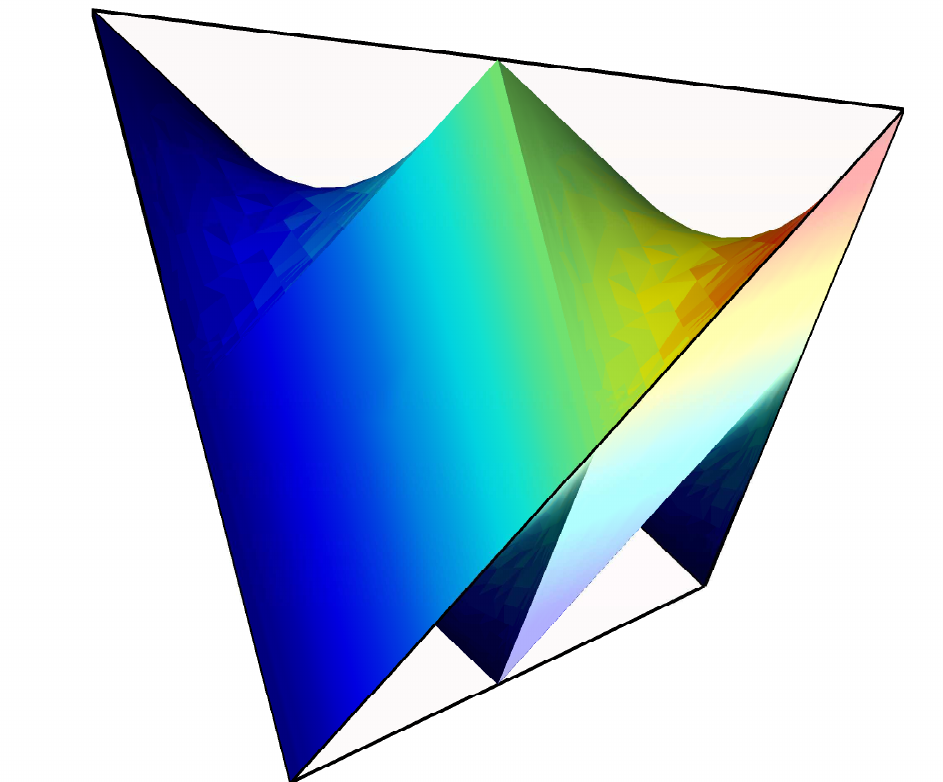}} 
\put(1.7,-.35){$\delta_{(00)}$}
\put(6,1.1){$\delta_{(11)}$}
\put(.1,6.8){$\delta_{(01)}$}
\put(7.7,5.8){$\delta_{(10)}$}
\put(4,6.5){$\tfrac12(\delta_{(01)}+\delta_{(10)})$}
\end{picture}
\end{center}
\caption{The set $u\cdot\Kcal_{1,2}=\{u \cdot K_{W,B}\colon W\in\R^{1\times2}, B\in\R^{1\times2}\}\subset\Delta_3$, where $u=(1/2,1/2)$. 
}\label{uK}
\end{figure}

The {\em input} distributions $q$ in eq.~\eqref{zonodbn} are restricted in different ways for each model:  
\vspace{-.2cm}
\begin{itemize}[leftmargin=1.3em]\setlength{\itemindent}{0em}
\item For DBNs $q\in \DBN(n_1,\ldots, n_l)$, and $\DBN(n_0,\ldots,n_l)=\DBN(n_1,\ldots,n_l)\cdot\Kcal_{n_1,n_0}$; in particular a DBN with layers of constant width is given by $\RBM_{n,n} \cdot \Kcal_{n,n}^{l-2}$. 
\item For directed RBMs $q\in\Mcal_m$, and $\RBM_{n,m}^{\text{dir}}= \Mcal_m\cdot\Kcal_{m,n}$. 
\item For RBMs $q\in \{\frac1Z\sum_{v}\exp( ( h  W +B)  v  +   C h )\colon C\in\R^m\}$. 
\item For DBMs $q\in\{\frac{Z_h}{Z} \sum_{h^2,\ldots,h^l} \prod_{k=1}^{l-1}\exp( ( h^{k+1} W^{k+1}+B^k) h^k ) \exp(  B^l h^l )\}$. 
\end{itemize}
In the case of RBMs and DBMs $q$ is subject to ``feedback'' from the visible units and depends on $W$ and $B$, while for DBNs and directed RBMs $q$ is independent from these parameters. 
The $2^m$ product distributions $p_{hW+B}$, $h\in\{0,1\}^m$, which we summarized in the rows of $K_{W,B}$,  however are the same for all these models. 
The smallest model which contains all models of the form $\Mcal\cdot\Kcal_{m,n}$ 
is the {\em $(m,n)$-zonoset mixture of products} (ZMP), defined by $\ZMP_{n,m}:=\Delta_{2^{m}-1}\cdot\Kcal_{m,n}$, or more explicitly: 
\begin{align}
\ZMP_{n,m} : = &\Big\{\sum_{h\in\{0,1\}^m} \!\!\!\lambda_h\, p_{hW+B} \;\Big|  \; \text{\small $\lambda_h\geq0, \sum_h\lambda_h=1$}, W\in\R^{m\times n}, B\in\R^n\Big\}\;.
\end{align} 
DBNs and DBMs are ``cut out'' from ZMPs by their specific constraints on the mixture weights $q(h)$. 
The mixture weights $q$ of $\DBN(n_0,n_1,\ldots,n_l)$ can be chosen arbitrarily and the model is  equal to $\ZMP_{n_0,n_1}$  only if $\DBN(n_1,\ldots,n_l)$ is a universal approximator on $\{0,1\}^{n_1}$. 

ZMPs are submodels of very large mixtures of products; $\ZMP_{n,m}\subseteq\Mcal_{n,2^m}$, for all $n$ and $m$. 
By results from~\cite{MontufarMorton2012}, $\Mcal_{n,2^m}$ is also the smallest mixture of products that contains $\RBM_{n,m}$ and thus $\ZMP_{n,m}$, when $4\lceil m/3\rceil\leq n$. 
On the other hand, each zonoset tuple shares the parameters $W$ and $B$, and the largest mixture of products contained in a ZMP is possibly relatively small. 
The total number of parameters of $\ZMP_{n,m}$ is $(m+1)n + 2^m-1$. 
We note that $\Mcal_{n,m+1}\subseteq\ZMP_{n,m}$ for all $n$ and $m$, and $\Mcal_{n,k}\not\subseteq\ZMP_{n,m}$ when $k>\frac{(m+1)n+2^m}{(n+1)}$ (by counting parameters).

\begin{example}
If the {\em input} $q$ is a point measure $\delta_{h}$, then the output is just the $h$th row  $q\cdot K_{W,B}=p_{hW+B}$ of the kernel. 
In particular $\delta_h\cdot \Kcal_{m,n}=\Mcal_{n}$ for any $h$. 
If the input is the uniform distribution $u$ on $\{0,1\}^m$, then the output $p=q\cdot K_{W,B}$ is the arithmetic mean of a zonoset tuple. 
Figure~\ref{uK} illustrates this set for one hidden  and two visible units. 
\end{example}

\section{Geometry and combinatorics of zonoset kernels}\label{section:geozono}

A {\em face} or a {\em cylinder set} of the $n$-cube is a maximal set of binary vectors of length $n$ with fixed values in a set of coordinates $I\subseteq[n]$. 
We write $[{h^\ast_I}]=\{h\in\{0,1\}^n \colon  h_I = h^\ast_I\}$ for the $(n-|I|)$-dimensional face with fixed values $h_i=h^\ast_i$ for all $i\in I$. 
We write $a \oplus_2 b$ for $a+ b \mod 2$. 
Given a vector $h\in\{0,1\}^m$ and a subset $I\subset[m]$, we write $h_I$ for a vector in $\{0,1\}^I$, or for the vector with entries $(h_I)_i=h_i$ if $i\in I$ and $(h_I)_i=0$ if $i\not\in I$. 
The support of a probability distribution $p$ defined on a set $\Xcal$ is $\operatorname{supp}(p):=\{x\in \Xcal\colon p(x)>0\}$.

We start showing that certain classes of kernels can be realized as zonoset kernels. 
Let $n=m$. 
Given any $p\in\Delta_{2^m-1}$, let $\operatorname{K}_p(h,v):=p(h\oplus_2 v )$. 
The rows of $\operatorname{K}_p$ are permuted versions of the probability distribution $p$. 
Figure~\ref{DBNlayerfig} illustrates the set of all kernels $\operatorname{K}_p$ with $p$ uniformly distributed on faces of $\{0,1\}^3$. 
The  {\em mixing times} of these kernels have been studied in the context of Markov chains on finite groups, see~\cite{kesten2003probability}. 

\begin{proposition}\label{markovkernels}
Let $p$ be any product distribution with support on any face of $\{0,1\}^n$ with fixed coordinates $I\subseteq[n]$, and let $\operatorname{K}_p(h,v)=p(h\oplus_2 v)$ for $v,h\in\{0,1\}^n$. 
Then there is a zonoset kernel $K_{W,B}\in\overline{\Kcal_{n,n}}$ with 
$K_{W,B}(h,v)=\operatorname{K}_p(h,h_{I^c}\oplus_2 v)$ for all $h,v\in\{0,1\}^n$, 
and in particular: 
\vspace{-.2cm}
\begin{itemize}[leftmargin=1.3em]\setlength{\itemindent}{0em}
\item 
$\supp (K_{W,B}(h,\cdot))=\supp (\operatorname{K}_p(h,\cdot))$ for all $h$, 
\item 
$K_{W,B}(h,\cdot)=\operatorname{K}_p(h,\cdot)$ for all $h$ with $\supp (h) \subseteq I$, e.g., for $h=(0,\ldots,0)$, 
\item 
If $p$ is uniformly distributed on a face of $\{0,1\}^n$, then $K_{W,B}=\operatorname{K}_p$. 
\end{itemize}
\end{proposition}

The following propositions show that the set $\Kcal_{m,n}$ has the dimension expected from parameter counting, and that its elements are generically full rank matrices. 
\begin{proposition}\label{toric}
The set of kernels $\Kcal_{m,n}$ is a multigraded toric variety. 
\end{proposition}

\begin{remark}
Let $V$ denote a sufficient statistics of the $n$-bit independence model, e.g., a matrix with columns the elements of $\{0,1\}^n$, and let $H$ be the $(m+1)\times 2^m$-matrix with columns $\{(1,h)\}_{h\in\{0,1\}^m}$. 
Consider the exponential family $\Ecal_{H\otimes V}$ with sufficient statistics $H\otimes V$ on $\Xcal=\{0,1\}^{n+m}$. Let $\Xcal_h=\{(v,h')\in\Xcal\colon h'=h\}\cong\{0,1\}^n$ for all $h$. 
For each $p\in\Ecal_{H\otimes V}$ there is a $K_{W,B}\in\Kcal_{m,n}$ (and vice versa) with 
$p(\cdot|\Xcal_h)= K_{W,B}(h,\cdot)\; \forall h\in\{0,1\}^m$. 
In particular, $\dim(\Kcal_{m,n})=(m+1)n$, as expected from counting parameters. 
\end{remark}

\begin{proposition}\label{symmetricZonosetkernel}
Assume that all rows of $W\in\R^{m\times n}$ are  multiples of the same vector $C\in\R^n$, i.e.,  $W=(\alpha_k C)_{k=1}^m$. 
For almost every $C$ and $(\alpha_1,\ldots,\alpha_m)\in\R^m$ the kernel $K_{W,B}$  is totally non-vanishing, i.e., all its minors are non-vanishing. 
\end{proposition}

\begin{proposition}\label{claim2}
For any $n$ and $m$ the kernels $K_{W,B}\in\Kcal_{m,n}$ are full rank for almost all choices of $W$ and $B$. 
In particular, almost every zonoset kernel $K_{W,B}\in\Kcal_{m,n}$ is injective when $m\leq n$, and $\dim(\Delta_{2^m-1}\cdot K_{W,B})=\min\{2^m-1, 2^n-1\}$. 
\end{proposition}

\begin{example}\label{RBMvsDRBM}
Consider an RBM and a directed RBM, both with $m$ hidden and $n$ visible binary units, $m\leq n$. 
For almost all fixed choices of  $W$ and $B$, 
the sets of probability distributions $\{\sum_hp_{hW+B}(v)\, \frac{Z_h}{Z} p_C(h)\colon C\in\R^m\}$ and $\{\sum_h p_{hW+B}(v)\, p_C(h) \colon C \in\R^m\}$ represented respectively by the two models as the bias of the hidden units vary, are almost everywhere different from each other (their intersection has dimension strictly less than $m$). 
When training DBNs, the DBN modules (directed RBMs) are commonly treated as RBMs. 
By this example, the probability distributions that can possibly be represented by the DBN modules almost never match the trained RBM distributions. 
\end{example}

The binary vectors $\{0,1\}^n$ are the vertices of the $n$-dimensional unit hypercube. 
We call {\em edge} a pair $\{x,y\}\subset\{0,1\}^n$ with $d_H(x,y)=1$,  where $d_H(x,y):=|\{i\in[n]\colon x_i\neq y_i\}|$ denotes  the Hamming distance between $x$ and $y$. 

\begin{proposition}\label{propo4}
Each of the following tuples of product distributions can be realized as a subset of  rows of a zonoset kernel with an appropriate choice of $W$ and $B$:
\vspace{-.2cm}
\begin{enumerate}[leftmargin=1.3em]\setlength{\itemindent}{0em}
\item \label{uno} 
If $\C\subset\{0,1\}^m$, $|\C|=m+1$ are affinely independent vectors (over $\R^m$), e.g., $\C$ is a Hamming ball of radius $1$ in $\{0,1\}^m$, then 
$\{p_{hW+B}\}_{h\in \C}$ are any $m+1$ product distributions. 

\item \label{dos} 
Let $\C$ be a $K$-dimensional face of the $m$-cube, $K\leq m$. 
The set $\{p_{hW+B}\}_{h\in \C}$ contains the uniform distributions on the (nonempty) intersections of any $K$ faces of the $n$-cube. 
\item \label{tres} Let $\lambda\subseteq[n]:=\{1,\ldots,n\}$ and $\Lambda\subseteq[m]$ with $|\lambda|=|\Lambda|=K$. Let $\C$ be a $K$-face of the $m$-cube with free coordinates $\Lambda$. 
$p_{hW+B}$ is the uniform distribution  on $\{x \colon x_\lambda = h_\Lambda \}$ for all $h\in\C$. 
Note that $\{x \colon x_\lambda = h_\Lambda \}_{h\in \C}$ is  a partition of $\{0,1\}^n$ into blocks of cardinality $2^{n-K}$. 
\item \label{cuatro} 
Let $m= n$ and let $\{h^{i+},h^{i-}\}$, ${i=1,\ldots,m}$ be $m$ disjoint edges of the $m$-cube. 
$p_{h^{i+}W+B}$ is any distribution supported on the edge $\{h^{i+}, h^{i+} \oplus_2+ \be_i\}$,
and $p_{h^{i-}W+B}$ is any distribution supported on the edge $\{h^{i-}, h^{i-}\oplus \be_i \}$, for all $i\in[m]$. 
Moreover, $p_{hW+B} =\delta_h$ for all $h\not\in\cup_i\{h^{i+},h^{i-}\}$. 
(This statement in fact summarizes~\cite[Theorems~1 and~2]{LeRoux2010}). 
\end{enumerate}
\end{proposition}

\begin{corollary}\label{corollarymix}
The model $\DBN(n,m,m)$ contains the mixture model $\Mcal_{n,m+1}$. 
In contrast,  $\RBM_{n,m}$ does not contain $\Mcal_{n,m+1}$, in general.  
\end{corollary}

\subsection{Patterns of modes in zonoset tuples}\label{subsection:modes}

In the following we elaborate on the sets of modes that can be realized jointly by rows of zonoset kernels, slightly extending results on RBMs and mixtures of products shown in~\cite{MontufarMorton2012}.

A {\em mode} of a probability distribution $p\in\Delta_{2^n-1}$ is  point $x\in\{0,1\}^n$ such that $p(x)>p(y)$ for all $y$ with $d_{H}(x,y)=1$. 
The set of {\em strong modes}~\cite{MontufarMorton2012} of  $p$  is $\{x\in\{0,1\}^n\colon p(x)\geq \sum_{d_{H}(x,y)=1} p(y)\}$. 
We denote by $\mathcal{H}_\C\subseteq\Delta_{2^n-1}$ the set of probability distributions with strong modes $\C$. 
An   {\em$n$-bit code} $\Ccal$ is just a subset of $\{0,1\}^n$. The {\em minimum distance} of $\Ccal$ is defined as $\min\{d_H(x,y)\colon x,y\in\Ccal, x\neq y\}$. 
Given a sign vector $s\in\{-,+\}^n$, the {\em $s$-orthant} of $\R^n$ is the set of all vectors in $\R^n$ with sign $s$. 
We identify sign vectors $\{-,+\}^n$ and binary vectors $\{0,1\}^n$ via $-\mapsto0$ and $+\mapsto1$.

\begin{proposition}\label{corollary1}\mbox{}
\vspace{-.2cm}
\begin{enumerate}[leftmargin=1.3em]\setlength{\itemindent}{0em}
\item\label{corollary1_1} 
Let $\C\subset\{0,1\}^n$ be a code of minimum distance two. 
If the model $ \ZMP_{n,m}$ contains a probability distribution with strong modes $\C$, then there is an $(m,n)$-zonoset with a point in every  $s$-orthant of $\R^n$, $s\in\Ccal$. 
\item\label{corollary1_3}  
If $\ZMP(n_0, n_1)$ contains probability distributions with $2^{n_0-1}$ strong modes, then $n_1\geq n_0-1$. 
In fact $n_1\geq n_0$, when $n_0$ is odd and larger than one. 
\item\label{corollary1_4}  
If $\ZMP(n_0, n_1)$ is a universal approximator of distributions from $\Delta_{2^{n_0}-1}$ with $n_0\geq7$, then $n_1\geq n_0$. 
\end{enumerate}
\end{proposition}
In particular, when  $n_0\geq3$, the DBNs with layers of widths $n_0>n_1>\cdots>n_l$ cannot represent distributions with $2^{n_0-1}$ strong modes.  
If $\DBN(n_0, n_1,\ldots, n_l)$ is a universal approximator with $n_1=n_0-1$ (and $n_0\leq 6$), then $\DBN(n_1, n_2,\ldots, n_l)$ is also a universal approximator. 

A {\em linear threshold code} (LTC) is a subset of $\{0,1\}^n$ that corresponds to the sign vectors of the points of a zonoset in $\RR^n$. 
Equivalently, an LTC is an admissible multi-labeling of the vertices of a hypercube by a collection of linear threshold functions.

\begin{proposition}\label{theorem3}
Let $\C\subseteq\{0,1\}^n, |\C|=2^{m}$ be a code of minimum distance two. 
Then  both $u\cdot\Kcal_{n,m}$ and $\Delta_{2^m-1}\cdot\Kcal_{n,m}$ contain a distribution with strong modes $\C$ iff $\C$ is a linear threshold code. 
\end{proposition}

\begin{proposition}\label{notinclpropodos}\mbox{}
\vspace{-.2cm}
\begin{itemize}[leftmargin=1.3em]\setlength{\itemindent}{0em}
\item 
If $4\lceil m/3 \rceil \leq n$, then $u\cdot\Kcal_{n,m}\cap\,\mathcal{H}_{n,2^m}\neq\emptyset$ and 
$\Mcal_{n,k}\supseteq u\cdot\Kcal_{n,m} \text{ iff }   k\geq2^m$.
\item  
If $4\lceil m/3 \rceil  >  n$, then  $u\cdot\Kcal_{n,m}\cap\,\mathcal{H}_{n,L}\neq\emptyset$, where $L:=\min\{2^{l} + m-l , 2^{n-1}\}$, $l:=\max\{l\in\NN \colon 4\lceil l/3 \rceil \leq n\}$,  and 
$\Mcal_{n,k}\supseteq u\cdot\Kcal_{n,m} \text{ only if  }  k\geq L.
$
\end{itemize}
\end{proposition}

\subsection{Submodels of DBNs from probability sharing}\label{sharing}

The idea of this subsection is to propagate the probability mass of distributions generated by the top RBM of a DBN across the network, in order to learn something about the visible  probability distributions at the bottom. 
This can be accomplished by describing the products of kernels $\Kcal_{n_{l-1},n_{l-2}}\cdot \Kcal_{n_{l-2},n_{l-3}} \cdots \Kcal_{n_1,n_0}$. 
For simplicity we shall consider layers of same width as the visible layer, $n$. 
In this case the propagation can be interpreted as a process in the graph of a hypercube. 

A kernel realizes sharing of probability from a state $a\in\{0,1\}^n$ to a state $b\in\{0,1\}^n$ if its $a$th row has non-vanishing $b$th entry. 
It is possible to share probability from $a$ to a collection of states $b^{(1)},\ldots,b^{(s)}$ in arbitrary ratios by a product of $l$ kernels iff the $a$th row of a product of kernels in $\Kcal_{n,n}^l$ can be made an arbitrary distribution on  $b^{(1)},\ldots,b^{(s)}$. 
In particular, since all rows of zonoset kernels are product distributions, probability sharing from one state to more than two states, in arbitrary rations, is not possible in one single DBN layer.

An {\em $l$-path} on the graph of the $n$-cube is a list $S$ of $l$ vectors in $\{0,1\}^n$ with subsequent elements differing in at most one bit, $S_1,\ldots,S_l\in\{0,1\}^n$, $d_H(S_{k},S_{k+1})\leq1$. 
An $n$-bit {\em Gray code} of length $l$ is a special $l$-path with different subsequent elements. 
The transition sequence $T$ of a path is the list of bit-indices where the subsequent elements differ from each other (possible empty).

Let $\Scal(\RBM_{n,m})$ denote the collection of support sets of all faces of the probability simplex $\Delta_{2^n-1}$, which are contained in $\overline{\RBM_{n,m}}$. 
It is  known that any union of $(m+1)$ edges of the $n$-cube is  is in $\Scal(\RBM_{n,m})$, see~\cite[Theorem~1]{Montufar2011}. 
Consider some $R\in \Scal(\RBM_{n,n})$ 
and a collection of  $l$-paths $S^i$ starting from $R$, such that at any time $1\leq t\leq l-1$ two paths change the same bit only if they are visiting neighboring points. 
We denote the  collection of all such sets by  
\begin{equation}
\Sbb_n^l := \big\{ \cup_{i\in R} S^{i} \big| \cup_i S^i_{1}=R\in \Scal(\RBM_{n,n}),  T^{i}_t\neq T^{j}_t \text{ unless } d_H(S^{i}_t,S^{j}_t)=1\big\}\;. \label{graydbns}
\end{equation}

The following result generalizes~\cite[Lemma~1, Theorems~1 and~2]{LeRoux2010} to DBNs with any  number of layers  of constant width:  

\begin{lemma}
\label{propo1}
The model $\DBN(n,\ldots,n)$  with $l$ hidden layers  contains any probability distribution with support in an element of $\Sbb_{n}^l$.   
\end{lemma}

For some elements of $\Sbb_n^l$ we find an explicit description: 
\begin{proposition}\label{propo2} 
If $n\geq N (2^k+k+1)$ and $l\geq 2^{2^k}$ for some $k\in\NN$, then $ \Sbb_{n}^l$ contains the union of $N$ arbitrary $(2^k+k+1)$-dimensional faces of the $n$-cube with disjoint free coordinates.  In particular, when $l\geq 2^n/2(n-\log(n))$, the entire state space $\{0,1\}^n$ is an element of $\Sbb_n^l$. 
\end{proposition}

\section{Expressive power and approximation errors of DBNs}\label{section:mainresults}

In this section we describe some submodels of DBNs explicitly, and use them to  bound the approximation errors of DBNs from above. 

Let $\varrho=\{A_1,\ldots,A_K\}$ be a partition of $\{0,1\}^n$. 
The {\em partition model} $\Mcal_\varrho$ is the set of all probability distributions with $p(x)=p(y)$ whenever $x$ and $y$ belong to the same block $A_i$ of the partition $\varrho$. 

The following  collects some results shown in the previous section: 
\begin{theorem}
\label{mainres}
Let $l\in\mathbb{N}$. Let $k$ be the largest natural number for which $l-1\geq 2^{2^k}$, and let $K= 2^k + k + 1 \leq n$. 
The model $\DBN(n,\ldots,n)$  with $l$ hidden layers   contains: 
\vspace{-.2cm}
\begin{itemize}[leftmargin=1.3em]\setlength{\itemindent}{0em}
\item \label{fith}
Any $p\in\Delta_{2^n-1}$ with support contained in an element of $\Sbb_{n}^l$. 
\item \label{secth}
Any partition model $\Mcal_\varrho$ with partition $\varrho=\{ [y_{\lambda}] \}_{y_\lambda\in\{0,1\}^K}$, $\lambda\subseteq[n], |\lambda|=K$. 
\end{itemize}
\end{theorem}

If $K\geq n$, then the $\DBN$ is a universal approximator, which is consistent with~\cite[Theorem~1]{Montufar2011}.

The Kullback-Leibler divergence from a point $p$ to a model  $\Mcal$ in $\Delta_{2^n-1}$ is defined as $D(p\|\Mcal):=\inf_{q\in\Mcal}D(p\|q)$, where $D(p\|q):=\sum_{x\in\{0,1\}^n}p(x)\log\frac{p(x)}{q(x)}$ is the divergence from $p$ to $q$. 
The {\em maximal KL-divergence}~\cite{MatusAy03:On_Maximization_of_the_Information_Divergence, NIPS2011_0307} from a partition model $\Mcal_\varrho$ with $2^K$ blocks of cardinalities $2^{n-K}$, as  given in the second item of Theorem~\ref{mainres}, is $\max_{p\in\Delta_{2^n-1}}D(p\|\Mcal_\varrho)=(n-K)$, see~\cite[Corollary~3.1]{NIPS2011_0307}. 
The {\em Dirichlet prior} on $\Delta_{2^n-1}$ with concentration parameter $\boldsymbol{\alpha}=(\alpha_x)_{x\in\{0,1\}^n}$ is $\Dir_{\ba}(p):=\frac{1}{\sqrt{2^n}}\frac{\Gamma(\sum_{x}\alpha_x)}{\prod_{x}\alpha_x} \prod_{x} p(x)^{\alpha_x-1}$ for all $p\in\Delta_{2^n-1}$, whereby the sums and products are over $x\in\{0,1\}^n$. 
If $p$ is drawn from this prior, then the expected approximation error  is, see~\cite[Theorem~4]{MonRauh2012}: 
\begin{equation}
\mathbb{E}[ D(p\|\Mcal_\varrho)]=
(n-K)\ln(2) + \sum_{x\in\{0,1\}^n} \frac{\alpha_x}{\sum_y\alpha_y} h(\alpha_x) 
 - \sum_{j=1}^{2^K} \frac{\sum_{x\in A_j} \alpha_x}{\sum_y\alpha_y}h(\sum_{x\in A_j} \alpha_x) 
\;,  \label{expectdpartition}
\end{equation} 
where $h(k):=1+\frac12+\cdots+\frac1k$ denotes the $k$th {\em harmonic number}.   

The  approximation error of a DBN is bounded from above by the approximation error of any of its submodels. 
If we use any of the partition models with $2^K$ blocks of cardinalities $2^{n-K}$, we get: 
\begin{theorem}\label{approxerrors}
Consider a DBN with $l$ hidden layers of width $n$. 
\vspace{-.2cm}
\begin{itemize}[leftmargin=1.3em]\setlength{\itemindent}{0em}
\item 
The maximal KL-approximation error of this model is bounded from above by 
\begin{equation*}
\max_{p\in\Delta_{2^n-1}}D(p\| \DBN) \leq n-K\;,\quad\text{where $K=2^k + k + 1 = \log(2l\log(l))$}. 
\end{equation*} 
\item 
The expected KL-approximation error is bounded from above by eq.~\eqref{expectdpartition}. 
In particular, if $p$ is drawn uniformly at random from the probability simplex $\Delta_{2^n-1}$, then the expected divergence $\mathbb{E}[D(p\|\DBN) ]$ is bounded from above by $1 +\ln(2^{n-K})- h(2^{n-K})$. 
\end{itemize}
\end{theorem}

\section{Discussion}\label{section:discussion}

Deep belief networks generate mixtures of tuples of product distributions  whose parameters 
are projections of hypercubes' vertices (zonosets), described by very few shared parameters. 
We cast these tuples of product distributions as the rows of stochastic matrices (zonoset kernels), and studied properties  such as their rank, symmetries, and combinatorics. 

This analysis exposes similarities of DBNs and DBMs, and shows possible ways of defining distributed mixtures of products; e.g., as $\Ecal\cdot\Kcal$, with a low-dimensional model $\Ecal\in\Delta_{2^m-1}$, and a family of kernels $\Kcal$. The rows of each kernel in the family $\Kcal$
 can be chosen as product distributions with parameters equal to the projected vertices of a hypercube, or  the projected vertices of any other low-dimensional polytope. 
In contrast, standard, unrestricted mixtures of products, correspond to projected vertices of (high-dimensional) simplices. 

Kernels are helpful for understanding probability sharing in layered networks. 
We showed explicit classes of probability distributions than can be learned by DBNs depending on the number of hidden layers that they contain. 
Various submodels of RBMs with $k$ parameters, such as unions of partition models, can be learned by deep and narrow DBNs with $k$ parameters. We showed that the maximal approximation error of narrow DBNs is not larger than the  upper bounds on the approximation errors of RBMs with the same number of parameters shown in~\cite{NIPS2011_0307}. 

Furthermore, we bounded the expected approximation error of DBNs from above. 
Our bounds are with respect to Dirichlet priors. These priors do not only have technical advantages, but  are  a canonical choice when no information is availble about the real distribution of the targets. 
It could be interesting to consider other priors  in future work. 
We note in particular, that the exact expected error formula from Theorem~\ref{approxerrors} item~2, eq.~\ref{expectdpartition}, can be integrated over an hyperprior of interest.

The approximation error bounds from Theorem~\ref{approxerrors} can possibly be improved by taking into account the totality of DBN submodels described in this paper, instead of just partition models. 
It is worth mentioning that any DBN which is a graphical supermodel of $\DBN(n_0, n_0-1, n_0-2,  \dots, 1)$ has the general Markov model corresponding to any tree on $n_0$ leaves as a graphical submodel.  That is, this DBN contains the union of all such tree models.  
Furthermore, DBNs often contain Hadamard  products of trees as well, so it is possible to study their dimension by {\em tropicalization}~\cite{tropical}.

\subsubsection*{Acknowledgments}
This work is supported in part by DARPA grant FA8650-11-1-7145.

\begin{small}
\bibliography{referenzen}{}
\bibliographystyle{abbrv}
\end{small}

\appendix
\section*{Proofs}

\subsection*{Geometry and combinatorics of zonoset kernels}

\begin{proof}[Proof of Proposition~\ref{markovkernels}]
The kernel $K_{h^\ast_I}\equiv K_{[u_{h^\ast_I}]}$ has rows equal to the indicator functions of $h\oplus_2 [h^\ast_I]$, $h\in\{0,1\}^n$, multiplied by the constant $2^{-(n-|I|)}$. 
Note that $[h^\ast_I]= \be_i\oplus_2 [h^\ast_I]$ for all $i\in[n]\setminus I$. 
For each $v_{[n]\setminus I}\in\{0,1\}^{[n]\setminus I}$, the sets $(v_I,v_{[n]\setminus I}) \oplus_2 [h^\ast_I]$, $v_I\in\{0,1\}^I$  partition  $\{0,1\}^n$ into $2^{|I|}$ cylinder sets.  
The connection weights $W(i,j)= \alpha (-h_i^\ast + \frac12) \delta_i(j)  \mathds{1}_I(j)$ and the bias weights $B(j)= -\alpha \frac12  (-h_j^\ast + \frac12) \mathds{1}_I(j)$ produce the kernel 
\begin{equation*}
K_{W,B}(h,v)=\exp(\alpha \frac12(-h^\ast_{I\cap\supp h}+\frac12\mathds{1}_{I\cap\supp h}  + h^\ast_{I\setminus\supp h}-\frac12\mathds{1}_{I\setminus\supp h} ) v)/Z .
\end{equation*}  
The limit $\lim_{\alpha\to\infty} K_{W,B}$ is equal to $K_{{[h^\ast_I]}}$. 
To complete the proof we add the natural parameter vector $C_{I^c}$ of $p$ to the previously defined bias vector $B$. 
Then $K_{hW+B+C_{I^c}}$ satisfies the claims.  
\end{proof}

\begin{proof}[Proof of Proposition~\ref{toric}]
Replacing the parameters $W_{ij}$, $B_{j}$ with their exponentials $\omega_{ij}$ and $\beta_j$, we obtain a multigraded monomial map $Q:\CC^{nm+n} \ra \prod_{i=1}^m \PP^{2^n-1}$; $q_{h,v} = \prod_{j=1}^n \beta_j^{v_j} \prod_{i=1}^m \omega_{ij}^{h_i v_j}$. 
The Zariski closure of the image of this map is a multigraded toric variety inside a product of $2^m$, $(2^n-1)$-dimensional projective spaces, one for each hidden state.  This variety is cut out by a multigraded monomial ideal generated by the multigraded binomials appearing in the kernel.    
\end{proof}

\begin{proof}[Proof of Proposition~\ref{symmetricZonosetkernel}]
The rows of $K_{W,B}$ are the product distributions with natural parameters the zonoset generated by $W$ and $B$. 
For assessing the rank of $K_{W,B}$ we may neglect the normalizing constants, and consider the matrix $\tilde{K}_{W,B}$ with rows $(\exp((h W+ B)v) )_{v\in\{0,1\}^n}$, $h\in\{0,1\}^m$. 
Furthermore, for any $B$ with finite entries, the rank of $\tilde{K}_{W,B}$ and $\operatorname{diag}(\exp(-Bv))_v \cdot\tilde{K}_{W,B}=\tilde{K}_{W,\boldsymbol{0}}$ is equal. 

Given the assumptions, the zonoset $\Zcal=\{h W +B\colon h\in\{0,1\}^m\}$ is contained in a straight line $\Zcal = \{\lambda_j C + B\}_{j=1}^{2^m}$, whereby the numbers $\lambda_j\in\R$ are all different from each other, for almost all $(\alpha_k)_k\in\R^m$. 
Let $(t_1,\ldots,t_{2^n}):=(\exp(C v))_{v\in\{0,1\}^n}$. 
Note that $t_i>0$ for all $i$, and all $t_i$ are different from each other, for almost all $C\in\R^n$. 
The rank of $\tilde{K}_{W,B}$ is equal to the rank of 
$(t_i^{\lambda_j})_{i,j}$, 
which, after some permutation of rows and columns, is a generalized Vandermonde matrix, known to be totally positive. Hence $\det(K_{W,B}(h,v))_{h\in H, v\in V}\neq0$ for all $H\subseteq\{0,1\}^m$ and $V\subseteq\{0,1\}^n$, as claimed. 
\end{proof}

\begin{proof}[Proof of Proposition~\ref{claim2}]
1) First note that there is an open subset $\Omega\subset\R^{m\times n}\times\R^n$ of parameters ${W,B}$ for which the kernels $K_{W,B}$ are full rank: 
Assume that the zonoset $\{hW+B\colon h\in\{0,1\}^m\}$ intersects $2^m$ orthants of $\R^n$, e.g., $W=I_n$ and $B=\tfrac12 (1,\ldots,1)$. 
Then $K_{\alpha W,\alpha B}$ is full rank for all $\alpha$ larger than some $a\in\R$,  because for  $\alpha\to\infty$ each row of $K_{\alpha W,\alpha B}$ converges to a different point measure. 
2) Now, by Proposition~\ref{toric} $\Kcal_{m,n}$ is a (toric, irreducible) variety for all $m, n\in\NN_0$. 
Let $l=\min\{m,n\}$. 
The set $H$ of rank-deficient matrices in $\CC^{2^l\times2^l}$, or in $\prod_{i=1}^m \PP^{2^n-1}$, is a hypersurface cut out by the vanishing of the determinant (which is a homogeneous polynomial on the matrix entries). 
Since $\Kcal_{l,l}\not\subseteq H$, 
by~\cite[Proposition~7.1]{hartshorne1977algebraic}, 
every irreducible component of $\Kcal_{l,l}\cap H$ has dimension $\dim(\Kcal_{l,l})-1$. 
This is also an upper bound for the dimension of the real part of the set of rank-deficient kernels in $\Kcal_{l,l}$. 
\end{proof}

\begin{proof}[Proof of Example~\ref{RBMvsDRBM}]
Both models have the same zonoset kernels. 
For any choice of $W$ and $B$, the set of inputs of the directed RBM are the product distributions $\Mcal_m$. The set of inputs of the RBM are the distributions $q(h)=\frac{Z_h}{Z}\cdot \exp(C h)$, which are product distributions iff $z(h)=\frac{Z_h}{\sum_h Z_h}\in\Delta_{2^m-1}$ is a product distribution. 
In~\cite{Cueto2010} it is shown that $\{\frac1Z\sum_v\exp(hWv+Bv)\colon W,B\}$ is a set of dimension $mn+n$ when $n ...m $. 
Since $K_{W,B}$ is injective, each output has a unique preimage. 
\end{proof}

\begin{proof}[Proof of Corollary~\ref{corollarymix}]
By item~\ref{uno} of Proposition~\ref{propo4}, 
the set of distributions $q\in\Delta_{2^m-1}$ with support on a radius-one Hamming ball is mapped by $\Kcal_{m,n}$ into the $(m+1)$-mixture of product distributions $\Mcal_{n,m+1}$. The claim follows using that $\RBM_{n,m}$ contains any $p$ with $|\supp(p)|\leq m+1$, see~\cite{Montufar2011}.
That RBMs do not contain the mixture model is a result from~\cite{MontufarMorton2012}. 
\end{proof}

\subsubsection*{Patterns of modes in zonoset tuples}
\begin{proof}[Proof of Proposition~\ref{corollary1}]\mbox{}
\vspace{-.2cm}
\begin{enumerate}[leftmargin=0em]\setlength{\itemindent}{1.3em}
\item 
The first item follows from~\cite[Theorems~3 and~11]{MontufarMorton2012}.  
\item
For the first part:  
The number of strong modes of a  mixture of $k$ binary product distributions is at most $k$~\cite[Theorem~3]{MontufarMorton2012}. 
For the second part: 
If $n$ is odd and larger than one, then the smallest mixture of binary product distributions whose natural parameters are a zonoset and which approximates $u_{Z_{\pm,n}}$ arbitrarily well, has a zonoset generated by at least $n$ vectors. 
See~\cite[Proposition~14]{MontufarMorton2012}. 
\item
The first part of the third item follows from parameter counting: The model $\sum_h \frac1Z\exp((hW+B)v) p(h)$, $p\in\Delta_{2^{n-1}-1}$ has a total of $n^2+n +2^{n-1}-1$ parameters. This number is smaller than $\dim(\Delta_{2^n-1})=2^n-1$ when $n\geq7$. 
For the second part: Any mixture of binary product distributions which approximates some $p$ with support ${Z_{\pm,n}}$ arbitrarily well, mixes the $2^{n-1}$ Dirac distributions $\delta_v$, $v\in Z_{\pm,n}$, see~\cite{Montufar2010a}. Hence if $\DBN(n_0,\ldots,n_l)$ approximates any distribution $p$ with support $Z_{\pm,n}$ arbitrarily well, then the mixture weights (distributions from  $\DBN(n_1^l)$) approximate $p|_{Z_{\pm,n}}$ arbitrarily well. 
\qedhere
\end{enumerate}
\end{proof}
\begin{proof}[Proof of Proposition~\ref{theorem3}] This is a direct consequence of the analysis from~\cite{MontufarMorton2012}.
\end{proof}

\begin{proof}[Proof of Proposition~\ref{notinclpropodos}]
The proof of the first item follows the lines of the proof of~\cite[Theorem~32]{MontufarMorton2012}. 
For the second item, note that if $\DBN(n,m,\ldots)$ can represent some $p$, then $\DBN(n,m+1,\ldots)$ can represent $\lambda p + (1-\lambda)\delta_x$ for any $x\in\{0,1\}^n$ for some $0<\lambda<1$. 
\end{proof}

\begin{proof}[Proof of Theorem~\ref{propo1}]
This result is a straightforward generalization of~\cite[Lemma~1, Theorems~1 and~2]{LeRoux2010}. 
The elements of  $\Sbb_n^l$ meet the conditions of these lemma and theorems by definition. 
\end{proof}

\subsection*{Submodels of DBNs from probability sharing}
\begin{proof}[Proof of Proposition~\ref{propo2}]\mbox{}
\vspace{-.2cm}
\begin{enumerate}[leftmargin=1.3em]\setlength{\itemindent}{0em}
\item 
This follows immediately from~\cite[Lemma~4]{Montufar2011}. Any sub-DBN with layers of width $(n-R)$ is contained in the DBN with layers of width $n$. 
The distribution on the states of the remaining $R$ visible nodes can be set to a point measure. 
\item 
This follows from a similar argument as the first item. 
Any set of cardinality $(n+1)$ is an $S$-set of $\RBM_{n,n}$. 
\qedhere
\end{enumerate}
\end{proof}

\begin{proof}[Proof of Proposition~\ref{propo4}]\mbox{}
\vspace{-.2cm}
\begin{itemize}[leftmargin=1.3em]\setlength{\itemindent}{0em}
\item[\ref{uno}]
If  $\C=\{h^{(0)},\ldots,h^{(m)}\}$ are affinely independent, then  
$\{h^{(1)}-h^{(0)},\ldots, h^{(m)}-h^{(0)}\}$ are linearly independent and can be mapped by $W$ to an arbitrary set $\{W'_1,\ldots,W'_m\}\subset\R^n$. 
Choosing $B=B' - h^{(0)} W$,  we can make $\{hW +B\colon h\in\C\}$ be arbitrary vectors $B', W'_{1},\ldots,W'_{m}$, and so, $\{p_{hW+B}\colon h\in\C\}$ is an arbitrary set of $m+1$ product distributions.

\item[\ref{dos}]
Any $h$ can be identified with its support set. 
$p_{\{i\}}$ $i\in[m]$ are $m$ uniform distributions on arbitrary faces $F_i$ of the $n$-cube. 
$p_{\lambda}$ is uniformly distributed with support $\operatorname{argmax}(\sum_{i\in\lambda} e_{F_i})$. 
E.g., if $F_\lambda:=\cap_{i\in\lambda} F_i\neq\emptyset$, then $\supp(p_\lambda)= F_\lambda$. 

\item[\ref{tres}] This follows from the choice $W_{:,\lambda}=\alpha I_\lambda$, the identity matrix, and $W_{:,[n]\setminus\lambda}=0$.

\item[\ref{cuatro}] 
Consider any $l\in[n]$. 
Consider a pair of vectors $\{x,y\}$ which is an edge of $\{0,1\}^m$. Let $r\in[m]$ be the entry where they differ. Let $s\in[m]$ be arbitrary. 
Denote by $\hat x$ the vector $\hat x_i = x_i\,\forall i\neq r,s$ and $\hat x_r=0$, $\hat x_s=0$. Denote by $\be_i$ the vector with one $1$ at the position $i$ and zeros else. By $\mathds{1}$ the vector of ones. 
Choosing 
\begin{eqnarray*}
W_{:,l} &=& \omega ( 2\hat x - \hat{\mathds{1}} + (1-2 x_s)m \be_s  + (p-q)\be_r ) \\
b_l  &=& -\omega (|\supp(x)|-1 + x_s m ) + q
\end{eqnarray*}
 yields in the limit $\omega\to\infty$ that $P(v_l=h_s | h\neq x,y ) =1$, $P(v_l = 1 | h=x ) = p$, and $P(v_l = 1 | h=y ) = q$, i.e.,   
\begin{eqnarray*}
 P(v_l| h\neq x,y) &=& \delta_{h_s}(v_l) \\
 P(v_l| h=x) &=& p(v_l) \\
 P(v_l| h=y) &=& q(v_l) \;.
\end{eqnarray*}
Consider the case $m = n$. 
Let $\{x^i,y^i\}_{i=1}^{m}$ be $m$ disjoint edges of $\{0,1\}^m$. Let $s^i=i\;\forall i\in[m]$. Consider any $l\in[n]$. From the above discussion we get
\begin{equation}
P(v|h=x^l) = \prod_{i=1}^n P(v_i|x^l) = \prod_{i\neq l} \delta_{x^l_{s^i}}(v_i) \cdot p^l(v_l)\;, 
\end{equation}
which is an arbitrary distribution with support on the edge given by fixing $v_i=x_i^l\,\forall i\neq l$.  For $h\not\in\cup_{i=1}^m\{x^i, y^i\}$ and $s^i=i \,\forall i$ we get 
\begin{equation}
P(v|h\neq x^l,y^l\, \forall l) = \prod_{i=1}^n P(v_i|h) = \prod_{i } \delta_{h_{s^i}}(v_i) = \delta_{h}(v)\;, 
\end{equation}
which is the point measure on $\{v=h\}$. 
\qedhere
\end{itemize}
\end{proof}

\begin{example}
Figure~\ref{kernels} gives an example of zonoset kernels $K_{W,B}=\operatorname{K}_p$ in $\Kcal_{4,4}$ for $p$ the uniform distributions on faces of $\{0,1\}^4$. 
\end{example}
\begin{figure}
\begin{center}
\setlength{\unitlength}{13cm}
\begin{picture}(1,1)(0,0)
\put(0,0){\includegraphics[trim=2cm 5.3cm 2cm 0cm, clip=true, width=\unitlength]{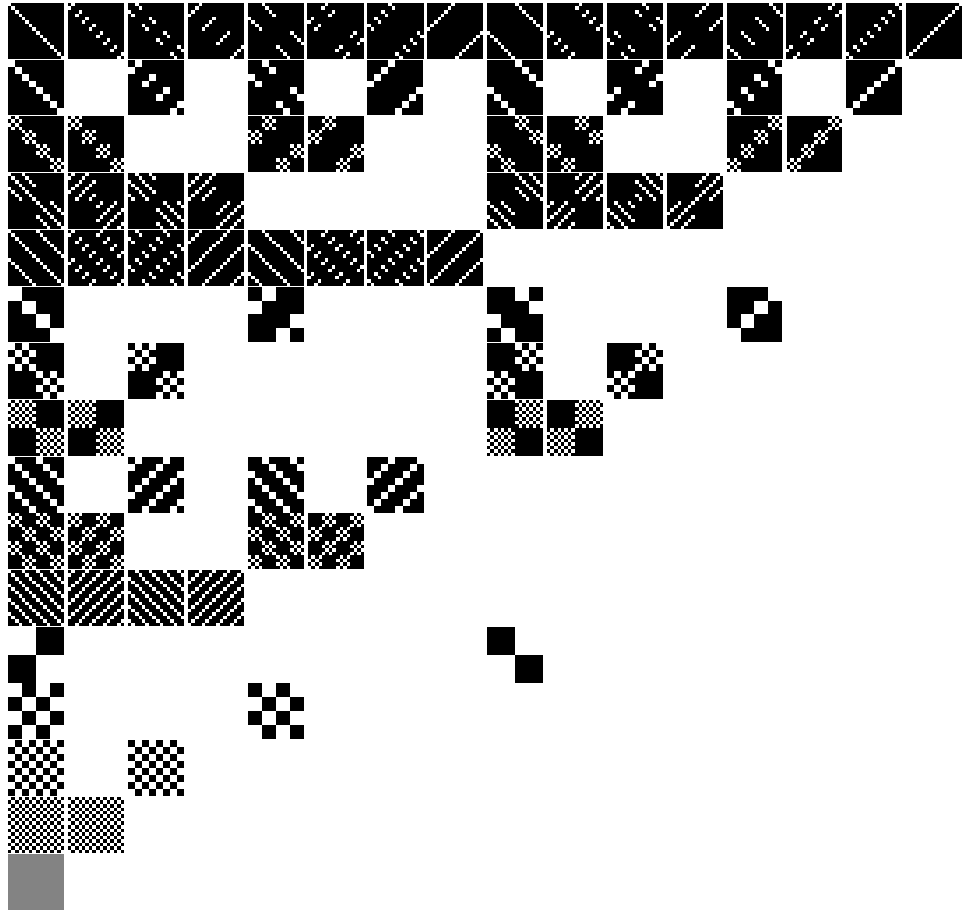}} 
\end{picture}

\end{center}
\caption{The kernels $\operatorname{K}_p$ for $p$ the uniform distributions on  faces of $\{0,1\}^4$ of dimension 
zero (first line), 
one (the next four lines; one line for each possible edge orientation), 
two (the next six lines; one for each pair in $\{1,2,3,4\}$), 
three (the next four lines), 
and four dimensional ($p$ is the uniform distribution on $\{0,1\}^4$). 
The first row of each kernel is always equal to the probability distribution $p$. 
The rows and columns of each kernel are in the lexicographical order of $\{0,1\}^4$. 
By Proposition~\ref{markovkernels}, all these kernels are contained in the family $\Kcal_{4,4}$. 
}\label{kernels}
\end{figure}

\end{document}